\documentclass[letterpaper]{article} 
\usepackage{aaai24}  
\usepackage{times}  
\usepackage{helvet}  
\usepackage{courier}  
\usepackage[hyphens]{url}  
\usepackage{graphicx} 
\usepackage{natbib}  
\usepackage{caption} 
 \usepackage{newtxmath}

\usepackage{algorithm}
\usepackage{algorithmic}

\usepackage{multicol,multirow}
\usepackage{amsmath}
\usepackage{booktabs}
\usepackage{color}

\newtheorem{proposition}{Proposition}
\newtheorem{proof}{Proof}

\usepackage{newfloat}
\usepackage{listings}
\DeclareCaptionStyle{ruled}{labelfont=normalfont,labelsep=colon,strut=off} 
\floatstyle{ruled}
\newfloat{listing}{tb}{lst}{}
\floatname{listing}{Listing}
\title{Double-Bounded Optimal Transport for Advanced Clustering and Classification}
\author{
    Liangliang Shi,
    Zhaoqi Shen,
    Junchi Yan\thanks{Correspondence author.}\\
}
\affiliations{
    Department of Computer Science and Engineering,\\
MoE Key Lab of Artificial Intelligence, Shanghai Jiao Tong University\\
\{shiliangliang, shenzhaoqi2271621336, yanjunchi\}@sjtu.edu.cn

}

\usepackage{bibentry}

\begin{document}

\maketitle

\begin{abstract}
Optimal transport (OT) is attracting increasing attention in machine learning. It aims to transport a source distribution to a target one at minimal cost. In its vanilla form, the source and target distributions are predetermined, which contracts to the real-world case involving undetermined targets. In this paper, we propose Doubly Bounded Optimal Transport (DB-OT), which assumes that the target distribution is restricted within two boundaries instead of a fixed one, thus giving more freedom for the transport to find solutions. Based on the entropic regularization of DB-OT, three scaling-based algorithms are devised for calculating the optimal solution. We also show that our DB-OT is helpful for barycenter-based clustering, which can avoid the excessive concentration of samples in a single cluster. Then we further develop DB-OT techniques for long-tailed classification which is an emerging and open problem. We first propose a connection between OT and classification, that is, in the classification task, training involves optimizing the Inverse OT to learn the representations, while testing involves optimizing the OT for predictions. With this OT perspective, we first apply DB-OT to improve the loss, and the Balanced Softmax is shown as a special case. Then we apply DB-OT for inference in the testing process. Even with vanilla Softmax trained features, our extensive experimental results show that our method can achieve good results with our improved inference scheme in the testing stage. 
\end{abstract}

\section{Introduction}



Optimal transport (OT)~\cite{cuturi2013sinkhorn} has been widely applied in machine learning~\cite{cui2019spherical,wang2013linear}. For instance, Wasserstein distance~\cite{arjovsky2017wasserstein,GulrajaniWGAN17} is applied with the dual form of OT to minimize the gap between the generated and real distributions via min-max adversarial optimization~\cite{li2022improving}. SwAV~\cite{caron2020unsupervised} employs the Sinkhorn algorithm for online clustering in self-supervised contrastive learning~\cite{chen2020simple,khosla2020supervised,wang2021understanding}.  OT-LDA~\cite{huynh2020otlda} devises a Wasserstein barycenter-based topic model, while in~\cite{xu2019gromov}, Gromov-Wasserstein distance is used for graph matching~\cite{WangICCV19,wang2021neural,hu2020open}. These works all assume the source and target distributions are fixed.

\begin{figure}[tb!]
    \centering
    \includegraphics[width=0.4\textwidth]{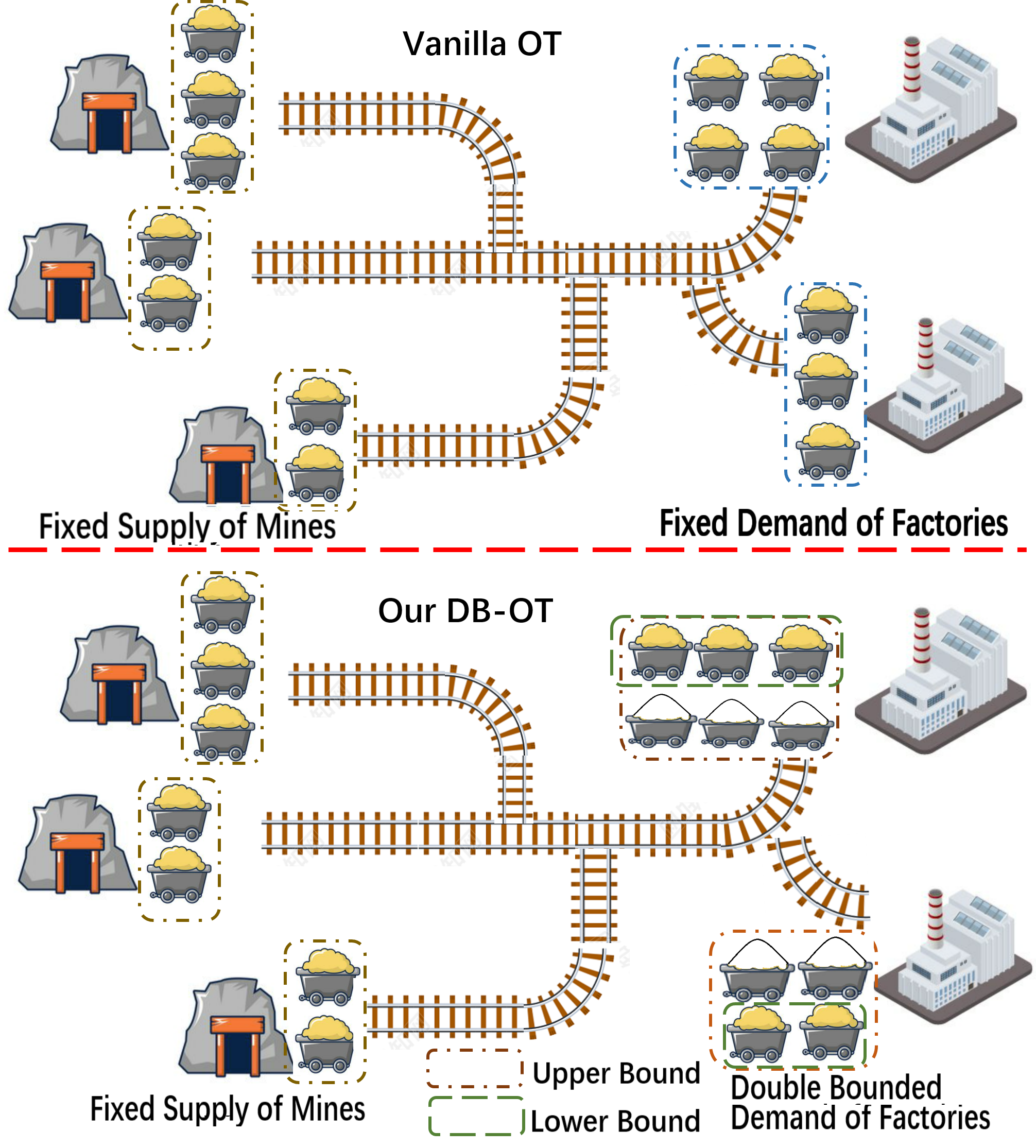}
    \caption{Illustration for the difference between vanilla OT and our DB-OT using the example of mines and factories as source and target, respectively. Vanilla OT assumes the equivalence between the supply and demand. In our DB-OT, we assume that the demand of the factory is not a fixed value, but rather a certain range by upper and lower bounds.
    } 
    \label{fig:difference}
\end{figure} 

However, in many real-world cases, the source or target distribution often varies over time and becomes uncertain. As shown in Fig.~\ref{fig:difference}, in vanilla OT, it assumes that the supply from mines and the demand from factories are fixed and and their total masses are equal to each other. While in our realistic setting we assume the demand from factories vary within a double-bounded range. When the demand exceeds the upper boundary, the factory reaches its production capacity limit. Conversely, when the demand falls below the lower boundary, the factory operates below its optimal capacity.\footnote{When both sides are not fixed, it degenerates to the lower bound. We discuss it in details in Sec.~\ref{sec:EDB-OT}.} 




Unbalanced OT~\cite{chizat2018scaling} is an OT variant to mitigate this inconsistency between source and target distributions. It relaxes the equality by only penalizing marginal deviation using divergences e.g. ${l}_1$, ${l}_2$, KL divergence, and etc. This method transforms the original hard constraints (equations) into penalty terms, allowing for a relaxation of the constraints and increasing the degree of freedom in transportation. However, the degree of relaxation can only be controlled via the penalty coefficient, which is not directly concerning the desired range of transportation results.





In this paper, we propose and term a new variant of OT called Double Bounded Optimal Transport (DB-OT), which allows the sum of coupling column to be in a range rather than a number, i.e., replacing the original equality constraint with two inequality constraints. Under this formulation, we further provide theoretical results with proof for DB-OT. Additionally, we propose three Sinkhorn-like algorithms based on entropic DB-OT. We also show that our results also inspire downstream applications.

Inspired by the OT-based topic model in OT-LDA~\cite{huynh2020otlda}, we propose a barycenter-based clustering method using DB-OT. This method allows for control of the sample quantity in each cluster, thereby avoiding the clustering of isolated or minority points within a cluster~\cite{wang2011improved}. Note that our barycenter-based clustering method is not limited to Euclidean space but can also be extended to other spaces, such as the Wasserstein space~\cite{agueh2011barycenters}. This extension allows our method to be generalized to handle more complex data, including text data.

Another application is using DB-OT for unbalanced image recognition. Specifically, to our best knowledge, we uncover the relation between classification and OT: training a classifier is equivalent to optimizing Inverse OT~\cite{shi2023understanding,stuart2020inverse,chiu2022discrete}, while inference in testing corresponds to optimizing OT problem. With this view, we first apply DB-OT to improve the loss for long-tailed data training. Different from traditional view using Softmax-based cross-entropy loss for learning, we increase the column-sum constraints for learning the representations and thus the Softmax generalizes to Sinkhorn-like iterations and the well-known balanced Softmax~\cite{ren2020balanced} can be a special case with only one iteration. Then we apply our DB-OT technique to inference in testing process and our approach is agnostic to the way how the classifier is trained. We successfully adopt our DB-OT inference technique to long-tailed, uniform, and reverse long-tailed datasets.
\textbf{The contributions of our work are:}

1) We propose DB-OT to handle the case that the targets of transportation fall in a double-bounded range, which is an important yet under-studied setting for practical machine learning, e.g. cluster size-controllable clustering and long-tailed  classification. We formulate DB-OT with the Kantorovich form by replacing the column equality constraints with double-bounded inequalities.

2) We solve DB-OT using entropic regularization and provide theoretical results, such as the static Schrödinger form, solution property, and dual form. To calculate the optimal solution, we propose a Bregman iterative algorithm based on the static Schrödinger form. Additionally, we develop matrix scaling-based methods using Lagrange methods, with two space-efficient variants derived from the primal and dual forms of the problem, respectively. Note that the motivation behind proposing these three algorithms is to align with the vanilla entropic OT for DB-OT.


3) We propose our barycenter-based method for clustering using DB-OT. It helps control the size of cluster and thereby avoiding the unwanted over-scattered clustering with isolated or very few samples in clusters. Besides, compared with the previous barycenter-based topic model, a reweighting modification is also proposed to avoid the bias for barycenter calculation especially in Euclidean space.

4) Finally, our DB-OT is employed in the long-tailed classification. From the OT perspective, we observe that training can be viewed as optimizing the Inverse Optimal Transport (IOT), while inference in the testing can be regarded as minimizing the OT itself. Building upon this insight, we apply DB-OT to enhance the classification loss during training and also utilize it in the testing process based on a trained model. 


\section{Preliminaries and Related Work}
\subsection{Basics of Optimal Transport}\label{sec:OTback}
As originated from~\cite{kantorovich1942transfer}, the Kantorovich's Optimal Transport is to solve a linear program, which is widely used for many classical problems such as matching~\cite{wang2013linear}. Specifically, given the cost matrix $\mathbf{C}$ and two histograms $(\mathbf{a},\mathbf{b})$, Kantorovich’s OT involves solving the coupling $\mathbf{P}$ (i.e., the joint probability matrix) by
\begin{equation}\label{eq:OT}
    \min_{\mathbf{P}\in U(\mathbf{a},\mathbf{b})} <\mathbf{C},\mathbf{P}>,  
\end{equation}
where  $U(\mathbf{a},\mathbf{b}) =\{\mathbf{P}\in R^+_{mn}|\mathbf{P}\mathbf{1}_n =\mathbf{a}, \mathbf{P}^\top \mathbf{1}_m =\mathbf{b}\}$. Relaxing with the entropic regularization~\cite{wilson1969use} is one of the simple yet efficient methods for solving OT, which can be formulated as~\cite{liero2018optimal}:
\begin{equation}\label{eq:OT_E}
    \min_{\mathbf{P}\in U(\mathbf{a},\mathbf{b})} <\mathbf{C},\mathbf{P}> - \epsilon H(\mathbf{P}),
\end{equation}
where $\epsilon>0$ is the coefficient for entropic regularization $H(\mathbf{P})$, and the regularization $H(\mathbf{P})$ can be specified as 
$
    H(\mathbf{P})=-<\mathbf{P},\log \mathbf{P} - \mathbf{1}_{m\times n}>.
$
The objective in Eq.~\ref{eq:OT_E} is $\epsilon$-{strongly} convex, and thus it has a unique solution, which can be solved by Sinkhorn algorithms as discussed in ~\cite{cuturi2013sinkhorn, benamou2015iterative}.

In this paper, beyond vanilla OT, we present our formulation for DB-OT and also propose the Sinkhorn algorithm.  


\subsection{Optimal Transport w/ Inequality Constraints}
The vanilla OT only considers the equality constraints in $U(\mathbf{a},\mathbf{b})$ yet inequality constraints are also handled in a few studies~\cite{caffarelli2010free, benamou2015iterative}. For instance, the optimal partial transport (OPT) problem~\cite{caffarelli2010free} assumes some mass variation or partial mass displacement should be handled for transportation. Thus OPT focuses on transporting only a fraction of mass $s\in[0,\min(||\mathbf{a}||_1,||\mathbf{b}||_1]$ as cheaply as possible. Then the constraints in this case is specified as:
\begin{equation}
     \mathcal{C}_{OPT}=\{\mathbf{P}\in R^+_{mn}|\mathbf{P}\mathbf{1}_n \leq\mathbf{a}, \mathbf{P}^\top \mathbf{1}_m \leq\mathbf{b},\mathbf{1}_m^\top\mathbf{P}\mathbf{1}_n =s\}.
\end{equation}

One can get the Partial-Wasserstein distance as $PW(\mathbf{a},\mathbf{b})=\min_{\mathbf{P}\in \mathcal{C}_{OPT}}<\mathbf{C},\mathbf{P}>$. This problem has been studied in~\cite{caffarelli2010free,figalli2010optimal,chapel2020partial}. In particular, \cite{chizat2018scaling,benamou2015iterative} propose the numerical solutions of OPT with entropic regularization. 
Different from the above constraints which only give upper bound in the constraints for mass transportation, in this paper, both upper and lower bounds are assumed for transportation, which allows the transportation results to fluctuate within a certain range rather than an exact number. 

\subsection{Unbalanced Optimal Transport}
OT in its vanilla form requires the two histograms $\mathbf{a}$ and $\mathbf{b}$ to have the same total mass. Many works have tried to go beyond this assumption and following the setting in Unbalanced OT~\cite{liero2018optimal, bai2023sliced}, the Kantorovich formulation in Eq.~\ref{eq:OT} is "relaxed" by only penalizing marginal deviation using some divergence $\mathcal{D}_\psi$: 
\begin{equation}\label{eq:UOT}
    \min_{\mathbf{P}\in R^+_P{mn}} <\mathbf{C},\mathbf{P}>+\tau_1\mathcal{D}_\psi(\mathbf{P}\mathbf{1}_n|\mathbf{a})+\tau_2\mathcal{D}_\psi(\mathbf{P}^\top\mathbf{1}_m|\mathbf{b})
\end{equation}
where $\tau_1$ and $\tau_2$ control how much mass variations are penalized as opposed to transportation of the mass and when $\tau_1=\tau_2\to+\infty$. This formulation equals to the original Kantorovich for $\sum_i \mathbf{a}_i=\sum_j \mathbf{b}_j$. When $\tau_1\to+\infty$ and $\tau_2<+\infty$, Unbalanced OT and DB-OT share a similar characteristic where the source distribution remains fixed while the target becomes unfixed. However, unlike DB-OT, it can only passively adjust the target distribution through $\tau_2$ lacking a direct mechanism to control the targets within a range.



\subsection{Unbalanced Image Recognition}

Unbalanced classification, particularly in the case of long-tailed recognition, is a well-known challenge that has garnered significant attention in vision and machine learning~\cite{zhang2023deep,he2009learning,lin2017focal}.

Various approaches~\cite{tan2020equalization,cui2019class,lin2017focal,kang2019decoupling,zhang2021deep} have been proposed to address this issue. One popular strategy involves rebalancing the class distribution in the training data~\cite{ren2020balanced,park2021influence}. Some methods employ techniques such as re-sampling~\cite{kang2019decoupling} or re-weighting~\cite{cui2019class} to ensure that the model pays more attention to minority classes during training. Data augmentation techniques, such as synthesizing additional samples for underrepresented classes, have also been utilized~\cite{kim2020m2m}. 
In this paper, we solve the unbalanced recognition with the OT perspective. We apply DB-OT both for learning the representations and inference in testing process, which is different to previous Bayesian view for classification.

\section{Double-Bounded Optimal Transport}
Vanilla OT usually involves two sets of equality constraints. Our motivation is to modify one type of constraints to the double bounded form which can be applicable to advanced clustering and classification.

\subsection{Formulation of DB-OT}\label{sec:EDB-OT}
We first assume the source measure $\alpha = \sum_{i=1}^m \mathbf{a}_i\delta_{x_{i}}$ and two bounds of target measure $\beta^d = \sum_{j=1}^n \mathbf{b}^d_i\delta_{y_{j}}, \beta^u = \sum_{j=1}^n\mathbf{b}^u_i\delta_{y_{j}}$ where $\mathbf{a}\geq \mathbf{0}$ and $\mathbf{b}^u\geq\mathbf{b}^d\geq\mathbf{0}$. Our purpose is to transport the source samples from $\alpha$ to the target measure which is defined between $\beta^d$ and $\beta^u$. By defining the cost matrix $\mathbf{C}$  between $\{x_{i}\}$ and $\{y_{i}\}$ (e.g. $\mathbf{C}_{ij}=d(x_i,y_j)$ where $d(\cdot,\cdot)$ is a distance), we can formulate the Double-Bounded Optimal Transport (DB-OT) as
\begin{equation}\label{eq:DB-OT}
\small{
    \min_{P\in \mathcal{C}(\mathbf{a},\mathbf{b}^u,\mathbf{b}^d)} <\mathbf{C},\mathbf{P}>=\sum_{i,j} \mathbf{C}_{ij}\mathbf{P}_{ij},
}
\end{equation}
where $\mathcal{C}(\mathbf{a},\mathbf{b}^u,\mathbf{b}^d)$ is the coupling set defined by constraints: 
\begin{equation}\label{eq:constraints}
    \mathcal{C}(\mathbf{a},\mathbf{b}^u,\mathbf{b}^d) = \{\mathbf{P}\in \mathbb{R}^+_{mn}|\mathbf{P}\mathbf{1}_n = \mathbf{a}, \mathbf{b}^d\leq\mathbf{P}^\top\mathbf{1}_m \leq \mathbf{b}^u\}.
\end{equation}
Similar to vanilla OT, the above model is still a linear program. From Eq.~\ref{eq:constraints}, we can see that we replace the original constraint $\mathbf{P}^\top\mathbf{1}_m =\mathbf{b}$ in Eq.~\ref{eq:OT} with $\mathbf{b}^d\leq\mathbf{P}^\top\mathbf{1}_m \leq \mathbf{b}^u$. That is, we slice $\mathbf{P}^\top\mathbf{1}_m$ between $\mathbf{b}^d$ and $\mathbf{b}^d$ to relax the constraint on $\mathbf{P}^\top\mathbf{1}_m$. It is obvious that when $\mathbf{b} = \mathbf{b}^d=\mathbf{b}^u$, DB-OT degenerates to vanilla OT in Eq.~\ref{eq:OT}. In this paper, we mainly focus on the entropic regularized formulation:
\begin{equation}\label{eq:EDB-OT}
    \min_{P\in \mathcal{C}(\mathbf{a},\mathbf{b}^u,\mathbf{b}^d)} <\mathbf{C},\mathbf{P}> -\epsilon H(\mathbf{P}),
\end{equation}
where $\epsilon$ is the regularization coefficient and $H(\mathbf{P})$ is the entropic regularization. It is obvious that the objective in Eq.~\ref{eq:EDB-OT} and the constraint set $\mathcal{C}(\mathbf{a},\mathbf{b}^u,\mathbf{b}^d)$ are both convex and thus Eq.~\ref{eq:EDB-OT} has a unique optimal solution. We will discuss its solution and the algorithm of entropic DB-OT in next subsection.

\noindent\textbf{Why not using double-bounded $\alpha$?} An intuitive question arises: why don't we assume that the $\alpha$ distribution is also constrained within upper and lower bounds? For instance, let's assume $\alpha^u = \sum_{i=1}^m \mathbf{a}^u_i\delta_{x_{i}}$ and $\alpha^d = \sum_{i=1}^m \mathbf{a}^d_i\delta_{x_{i}}$, and the coupling satisfies $\mathbf{a}^d\leq \mathbf{P}\mathbf{1}_n\leq \mathbf{a}^u$. In practice, the optimal transportation tends to transport mass vertically and towards the smaller lower bound. Without loss of generality, let's assume $\sum_i\mathbf{a}^d_i<\sum_j\mathbf{b}^d_j$. In this case, the optimal solution must satisfy $\sum_{ij}\mathbf{P}_{ij}=\sum_i\mathbf{a}^d_i$ (proof provided in oneline Appendix). As a result, setting an upper bound for the source distribution does not hold much significance.


\subsection{Sinkhorn Algorithm Variants for DB-OT}
In this subsection, we introduce several properties of DB-OT and present three corresponding Sinkhorn algorithms, namely the Bregman iterative algorithm, Sinkhorn\_Knopp algorithm, and Dual algorithm for DB-OT, respectively, which aims to be \textbf{in line with} the algorithms of \textbf{vanilla OT}. Although these algorithms are derived from different formulations or properties of DB-OT, they are fundamentally interconnected and converge to same solutions, which validates the effectiveness of the algorithms.



\noindent\textbf{$\blacksquare$ Variant-I: Bregman Iterations For DB-OT\\}
Similar to the vanilla entropic OT, our entropic DB-OT in Eq.~\ref{eq:EDB-OT} can {also be reformulated to the "static Schrödinger form" ~\cite{leonard2012schrodinger}}, which exactly learns a projection under KL divergence. 
\begin{proposition}[\textbf{Static Schrödinger Form}]\label{Prop:KLproof}
Redefine a general KL divergence in line with~\cite{benamou2015iterative}
$
    \widetilde{KL}(\mathbf{P}|\mathbf{K})=\sum_{ij}\mathbf{P}_{ij}\log\frac{\mathbf{P}_{ij}}{\mathbf{K}_{ij}}-\mathbf{P}_{ij}+\mathbf{K}_{ij}.
$
Let $\mathbf{K}_{ij}=e^{-\mathbf{C}_{ij}/\epsilon}$, the optimization in Eq.~\ref{eq:EDB-OT} is equivalent to the minimization:
\begin{equation}\label{eq:static}
\small{
        \mathbf{P}^*=\mathop{\arg\min}\limits_{\mathbf{P}\in \mathcal{C}(\mathbf{a},\mathbf{b}^u,\mathbf{b}^d)} \widetilde{KL}(\mathbf{P}|\mathbf{K}).
}
\end{equation}
\end{proposition}
Following~\cite{benamou2015iterative}, a new variant of Sinkhorn algorithm can be devised with iterative Bregman projections for DB-OT. By splitting the constraint set as
$
    \mathcal{C}(\mathbf{a},\mathbf{b}^u,\mathbf{b}^d) = \mathcal{C}_1\cap\mathcal{C}_2\cap\mathcal{C}_3
$
where 
$
        \mathcal{C}_1 = \{\mathbf{P}\in \mathbb{R}^+_{nm}|\mathbf{P}\mathbf{1}_n =\mathbf{a}\}, 
        \mathcal{C}_2 = \{\mathbf{P}\in \mathbb{R}^+_{nm}| \mathbf{P}^\top \mathbf{1}_m \geq\mathbf{b}^d\}, 
        \mathcal{C}_3 = \{\mathbf{P}\in \mathbb{R}^+_{nm}| \mathbf{P}^\top \mathbf{1}_m \leq\mathbf{b}^u\}.
$
We can get the KL projection for $\mathbf{P}$ under the constraints as:
\begin{equation}\label{eq:bregman_2}
\small{
\begin{aligned}
    & Proj_{\mathcal{C}_1}^{KL}(\mathbf{P})=\text{diag}\left(\frac{\mathbf{a}}{\mathbf{P}\mathbf{1}_n}\right)\mathbf{P}\\
    & Proj_{\mathcal{C}_2}^{KL}(\mathbf{P})=\mathbf{P}\text{diag}\left(\max\left(\frac{\mathbf{b}^d}{\mathbf{P}^\top\mathbf{1}_m},\mathbf{1}_n\right)\right)\\
    & Proj_{\mathcal{C}_3}^{KL}(\mathbf{P})=\mathbf{P}\text{diag}\left(\min\left(\frac{\mathbf{b}^u}{\mathbf{P}^\top\mathbf{1}_m},\mathbf{1}_n\right)\right),\\
\end{aligned}
}
\end{equation}
where the above division operator 
between two vectors is to be understood entry-wise (See details in online Appendix). 
Finally, as introduced by \cite{benamou2015iterative}, assuming $\mathcal{C}_l=\mathcal{C}_{l+3}$ with positive integer $l$ as the index of Bregman iteration, the minimization in Eq.~\ref{eq:static} can be solved by with the iterative projection 
$
    \mathbf{P}^{(n)} = Proj_{\mathcal{C}_n}^{KL}\left( \mathbf{P}^{(n-1)}\right),
$
starting from $ \mathbf{P}^{(0)}=\mathbf{K}=e^{-\mathbf{C}/\epsilon}$. Then we can get the solution by $\mathbf{P}^*=\lim_{n\to \infty}Proj_{\mathcal{C}_n}^{KL}( \mathbf{P}^{(n-1)})$, which exactly do the iterations for Eq.~\ref{eq:bregman_2} and \cite{bregman1967relaxation} shows its convergence guarantee.

\noindent\textbf{$\blacksquare$ Variant-II: Sinkhorn-Knopp algorithm for DB-OT\\}
In addition to analyzing DB-OT with KL divergence and Bergman iterative projections, Lagrangian methods can also solve the problem. Following~\cite{cuturi2013sinkhorn}, we show that the form of solution is specified as follows.
\begin{proposition}[\textbf{Solution Property}]\label{prop:uKv}
The optimal solution of the optimization in Eq.~\ref{eq:EDB-OT} is unique and has the form 
\begin{equation}
    \mathbf{P}^* = \text{diag}(\mathbf{u})\mathbf{K}\text{diag}(\mathbf{q}\odot\mathbf{v})
\end{equation}
where $\mathbf{q}\odot\mathbf{v}=(\mathbf{q}_i\mathbf{v}_i)\in \mathbb{R}_n^+$ and the three scaling variables $(\mathbf{u},\mathbf{q},\mathbf{v})$ satisfy $\mathbf{u}\in R^+_m$, $\mathbf{0}_n\leq\mathbf{v}\leq\mathbf{1}_n$ and $ \mathbf{q}\geq \mathbf{1}_n$.
\end{proposition}
The proof is given in uploaded online appendix. Comparing to vanilla OT with equality constraints, three scaling variables need to be iterated. We then also derive the Sinkhorn\_Knopp algorithm of DB-OT given the iteration number $l$:
\begin{equation}\label{eq:sinkhorn2}
\small{
\begin{aligned}
    &\mathbf{u}^{(l+1)} =  \frac{\mathbf{a}}{\mathbf{K}(\mathbf{v}^{(l)}\odot\mathbf{q}^{(l)})},  \\
    & \mathbf{q}^{(l+1)} = 
    \max\left(\frac{\mathbf{b}^d}{(\mathbf{u}^{(l+1)}\mathbf{K})\odot\mathbf{v}^{(l)}},\mathbf{1}_n\right),\\
    &\mathbf{v}^{(l+1)} =  \min\left(\frac{\mathbf{b}_u}{(\mathbf{u}^{(l+1)}\mathbf{K})\odot\mathbf{q}^{(l+1)}},\mathbf{1}_n\right).
\end{aligned}
}
\end{equation}
By initializing the scaling variables $\mathbf{0}_n\leq\mathbf{v}^{(0)}\leq\mathbf{1}_n$ and $\mathbf{q}^{(0)}\geq \mathbf{1}_n$, we can iterate by Eq.~\ref{eq:sinkhorn2} to compute $(\mathbf{u}^{(1)},\mathbf{q}^{(1)},\mathbf{v}^{(1)}), (\mathbf{u}^{(2)},\mathbf{q}^{(2)},\mathbf{v}^{(2)}),\dots$ until convergence to $(\mathbf{u}^{*},\mathbf{q}^{*},\mathbf{v}^{*})$. Then the solution can be obtained as $\mathbf{P}^* = \text{diag}(\mathbf{u}^{*})\mathbf{K}\text{diag}(\mathbf{q}^{*}\odot\mathbf{v}^{*})$. 


\noindent\textbf{$\blacksquare$ Variant-III: The Dual Sinkhorn algorithm\\}
We also analyze entropic DB-OT with dual formulation, which is devised as follows.
\begin{proposition}[\textbf{Dual Formulation}]\label{Prop:dual}
From the  optimization in Eq.~\ref{eq:EDB-OT}, we can get its dual formulation by maximizing
\begin{equation}
     L = <\mathbf{f},\mathbf{a}>+<\mathbf{g},\mathbf{b}^d> +  <\mathbf{h},\mathbf{b}^u> - \epsilon B(\mathbf{f},\mathbf{g},\mathbf{h})
\end{equation}
where $\mathbf{f}\in \mathbb{R}^n$, $\mathbf{g}\leq  \mathbf{0}$ and $\mathbf{h}\geq\mathbf{0}$ are the corresponding dual variables and $B(\mathbf{f},\mathbf{g},\mathbf{h}) = <e^{\mathbf{f}/\epsilon},\mathbf{K}e^{(\mathbf{g}+\mathbf{h})/\epsilon}>$.
\end{proposition}

Then we can get the new algorithm with a block coordinate ascent on the dual:
\begin{equation}
\small{
    \begin{aligned}
        &\Delta_\mathbf{f} L=\mathbf{a}-e^{\mathbf{f}/\epsilon}\odot\mathbf{K}e^{(\mathbf{g}+\mathbf{h})/\epsilon}\\
        &\Delta_\mathbf{g} L=\mathbf{b}^d-e^{(\mathbf{g}+\mathbf{h})/\epsilon}\odot\mathbf{K}^\top e^{\mathbf{f}/\epsilon}\\
        &\Delta_\mathbf{h} L=\mathbf{b}^u-e^{(\mathbf{g}+\mathbf{h})/\epsilon}\odot\mathbf{K}^\top e^{\mathbf{f}/\epsilon}
    \end{aligned}.
}
\end{equation}
Starting from arbitrary $\mathbf{g}^{(0)}\leq 0$ and $\mathbf{h}^{(0)}\geq 0$, we get the dual algorithm by iterating the following with iteration number $l$:
\begin{equation}\label{eq:dual}
\small{
    \begin{aligned}
    &\mathbf{f}^{(l+1)} = \epsilon\log\mathbf{a}-\epsilon\log\left(\mathbf{K}e^{(\mathbf{g}^{(l)}+\mathbf{h}^{(l)})/\epsilon}\right)\\
    &\mathbf{g}^{(l+1)} = \max\left\{\epsilon\log\mathbf{b}^d-\epsilon\log(\mathbf{K}^\top e^{\mathbf{f}^{(l+1)}/\epsilon})-\mathbf{h}^{(l)},\mathbf{0}\right\}\\
    &\mathbf{h}^{(l+1)} = \min\left\{\epsilon\log\mathbf{b}^u-\epsilon\log(\mathbf{K}^\top e^{\mathbf{f}^{(l+1)}/\epsilon})-\mathbf{g}^{(l+1)},\mathbf{0}\right\}\\
    \end{aligned}.
}
\end{equation}
The iteration continues till convergence to $(\mathbf{f}^*,\mathbf{g}^*,\mathbf{h}^*)$ and the optimal solution equals $\mathbf{P}^*=\text{diag}(e^{\mathbf{f}/\epsilon})\mathbf{K}\text{diag}(e^{(\mathbf{g}^*+\mathbf{h}^*)/\epsilon})$.

Exactly, Eq.~\ref{eq:dual} are mathematically equivalent to the Sinkhorn iterations in Eq.~\ref{eq:sinkhorn2} when considering the primal-dual relations. Indeed, we recover that at any iteration:
$
        (\mathbf{f}^{(l)},\mathbf{g}^{(l)},\mathbf{h}^{(l)}) = \epsilon\left(\log(\mathbf{u}^{(l)},\log\mathbf{q}^{(l)},\log\mathbf{v}^{(l)}\right).
$
In this paper, we utilize the Sinkhorn\_Knopp algorithm for clustering tasks and the Bregman iterations for long-tail classification.




\begin{figure*}[tb]
    \centering
    \includegraphics[width=0.82\textwidth]{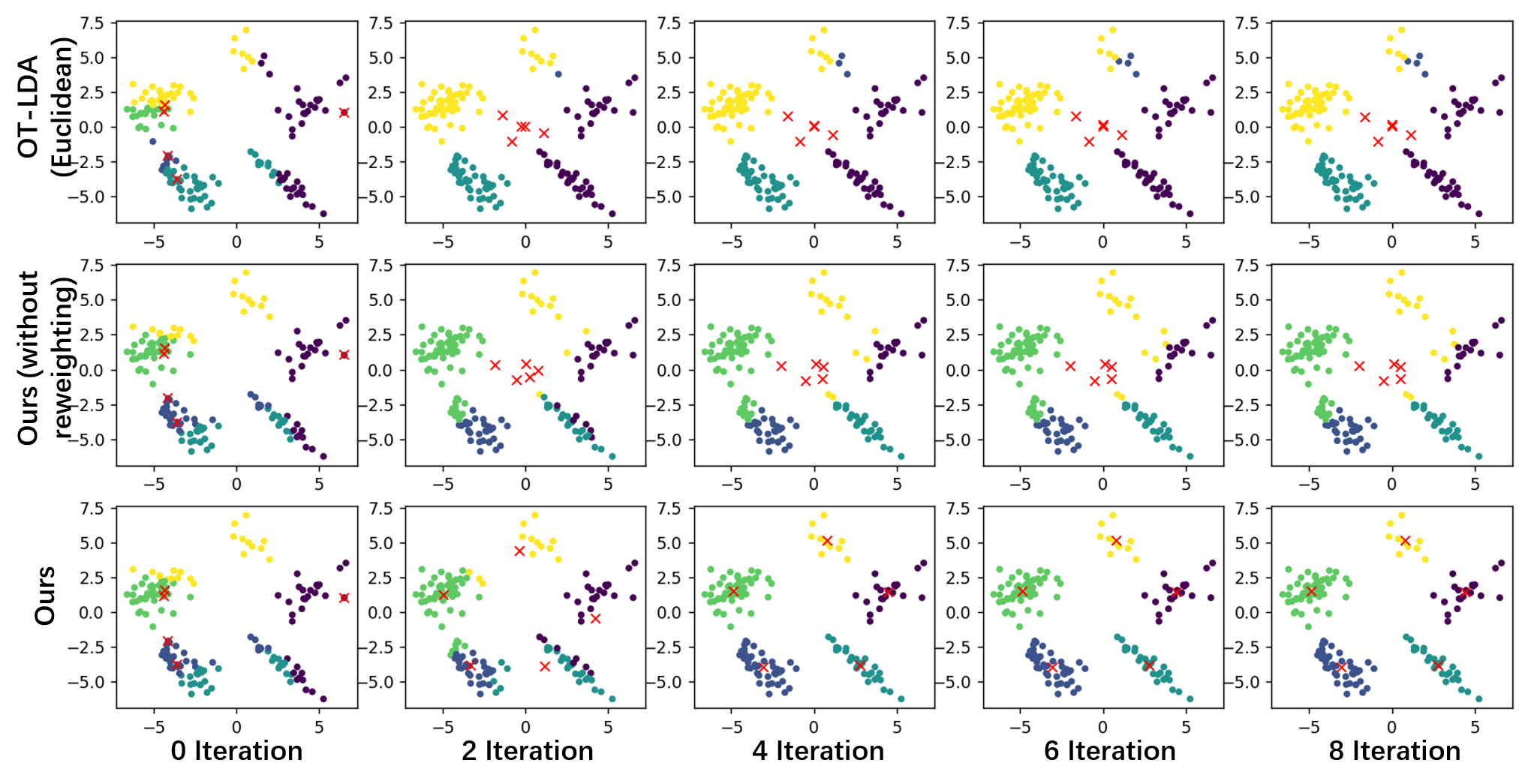}
    \caption{The results of Barycenter-based clustering, which  is performed on data points sampled from 5 Gaussian distributions. The colors represent the cluster assignments of the samples, and the red crosses denote the centroids/barycenters. Note that in both OT-LDA and our method without reweighting the barycenter weights, the calculated centroids exhibit a noticeable bias. } 
    \label{fig:clustering}
\end{figure*}

\section{DB-OT for Clustering and Classification}

\subsection{Barycenter-based Clustering}

Inspired by the OT-based topic model~\cite{huynh2020otlda}, we show the application of DB-OT to an advanced clustering task whereby the number of samples in each cluster is softly required to fall into a certan range. Specifically given samples $\{\alpha_s\}$, it learns the coupling $\mathbf{P}$ to match the samples to centroids,  and the clustered centroids $\{\beta_t\}$ are viewed as barycenters which needs to optimize. The optimization is specified as:
\begin{equation}\label{eq:EDB-OT_cluster}
\small{
    \min_{\mathbf{P}\in \mathcal{C}(\mathbf{a},\mathbf{b}^d,\mathbf{b}^u),\{\beta_t\}} \sum_{st} \mathbf{P}_{st}D(\alpha_s,\beta_t)-\epsilon H(\mathbf{P}),
}
\end{equation}
where $\mathbf{a}=\mathbf{1}$ (here we break the probability measure assumption as done in \cite{saad2021graph}) and $\mathbf{b}^d,\mathbf{b}^u$ are used to denote the given bounds of clustered centroids i.e. the minimum/maximum "quantity" of samples in each cluster. $D(\alpha_s,\beta_t)$ is the distance between $\alpha_s$ and $\beta_t$ (e.g. Euclidean or Wasserstein distance), and $ H(\mathbf{P})$ is the entropic regularization which aims to relax the solution of coupling. This problem can be solved by alternating optimization. For each iteration, we first fix the centroid distributions $\{\beta_t\}$ to compute the distance $D(\alpha_s,\beta_t)$ between sample $i$ and centroid $j$ and then transportation matrix $\mathbf{P}$ can be learned from the source samples to clustered centroids. 

Note in our method, we have $\mathbf{b}^d\leq\mathbf{P}^\top \mathbf{1}_n\leq \mathbf{b}^u$, which means the coupling between samples and cendroids are controlled within a desired range, avoiding clusters with isolated or very few samples, as well as dominating clusters.

Specifically, for solving the optimization to clustering in Eq.~\ref{eq:EDB-OT_cluster}, we first initialize $\{\beta_t\}$ with initialization of Kmeans++ or set them as randomly from $\{\alpha_s\}$. Then we can learn by iterating the following two steps:

\noindent\textbf{1) Fixing $\{\beta_t\}$ to compute $\mathbf{P}$.}  If the barycenter $\{\beta_t\}$ is known, then one can calculate the matrix $\mathbf{D}$ where $\mathbf{D}_{st} = D(\alpha_s,\beta_t)$. Viewing the matrix $\mathbf{D}$ as the cost function, the optimization of $\mathbf{P}$ in Eq.~\ref{eq:EDB-OT_cluster} is equal to DB-OT optimization in Eq.~\ref{eq:EDB-OT} and the algorithm proposed in Sec.~\ref{sec:EDB-OT} can be used to get the result of $\mathbf{P}$. Here we adopt the Sinkhorn\_Knopp algorithm as given in Eq.~\ref{eq:sinkhorn2}.

\noindent\textbf{2) Fixing $\mathbf{P}$ to compute $\{\beta_t\}$}. Assuming the transportation $\mathbf{P}$ is known,  our goal is now to update $\{\beta_t\}$ as the solution of the following optimization problem
\begin{equation}\label{eq:reweightBC}
    \beta_t = \arg\min_\beta\sum_s\frac{\mathbf{P}_{st}\cdot R_{st}}{\sum_{s'} \mathbf{P}_{s't}\cdot R_{s't}}D(\alpha_s,\beta)
\end{equation}
where $R_{st}=1$ if $t=\arg\max_{t'}  \mathbf{P}_{st'}$  otherwise $R_{st}=0$ given the sample $\alpha_s$. Here $R_{st}$ are used for {re-weighting} the weights of barycenters, which aims to debias the influence for samples from different clusters. Fig.~\ref{fig:clustering} is the clustering results and we can find a large bias for cendroids without re-weighting the barycenter weights. 

\noindent\textbf{Under Different Metric Space for $D(\cdot,\cdot)$.} 
Clustering can be performed in different metric spaces by setting the distance function $D(\cdot,\cdot)$ to adapt to various metric settings or datasets. For example, when $D(\cdot,\cdot)$ represents the Euclidean distance and ${\alpha_s}$ corresponds to feature points in Euclidean space, we aim to cluster ${\alpha_s}$ in the Euclidean space. The centroids ${\beta_t}$ can be directly computed as follows:
\begin{equation}
    \beta_t = \sum_s\frac{\mathbf{P}_{st}\cdot R_{st}}{\sum_{s'} \mathbf{P}_{s't}\cdot R_{s't}}\alpha_s.
\end{equation}
Exactly if we set $R_{st}=1$ for all $s$ and $t$ (i.e. without reweighting as shown in the second row of Fig.~\ref{fig:clustering}), the barycenters can be easily influenced by samples from other clusters. This can lead to a bias in the barycenter calculation, particularly when the value of $\epsilon$ is not sufficiently small.
\begin{figure*}[tb!]
    \centering
    \includegraphics[width=0.85\textwidth]{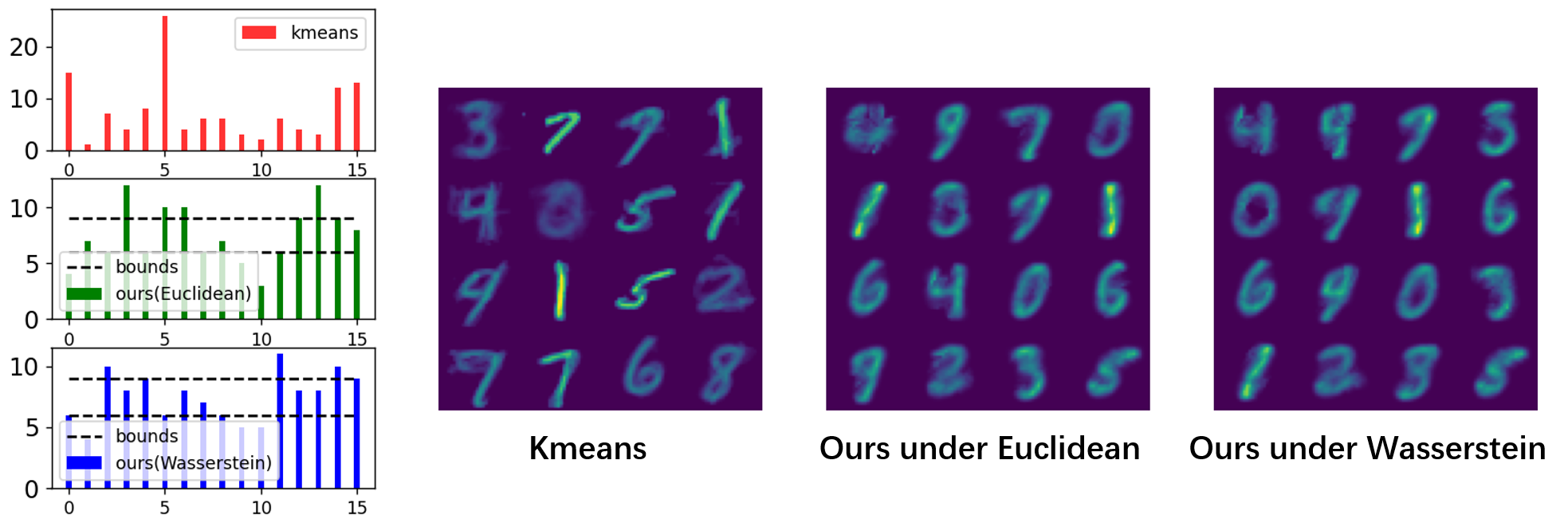}
    \caption{Clustering distribution and the pixel-wise mean centroids (forming into numbers) on MNIST. Our results are well controlled within the bounds, and kmeans cannot satisfy this property resulting in more scattered clusters of varying size.} 
    \label{fig:Barycenter_MNIST}
\end{figure*} 

In addition to Euclidean space, clustering can also be performed in Wasserstein space. In this case, we assume that ${\alpha_s}$ is no longer a set of feature points but a collection of probability measures, i.e., ${\alpha_s} = \sum_i \mathbf{a}^s_i\delta_{x_i}$. In this setting, the computation of the barycenter in Eq.~\ref{eq:reweightBC} does not have an analytical solution but requires iterative methods, as proposed in \cite{benamou2015iterative} for barycenter calculation.

\noindent\textbf{Advantages for our Barycenter-based Clustering.} Our barycenter-based clustering is essentially a unified clustering framework. Compare to the previous clustering works (e.g. Kmeans~\cite{ahmed2020k} or OT-LDA~\cite{huynh2020otlda}), the advantages  can be summarized as follows: 
\textbf{1) Our} clustering methods apply DB-OT to calculate the matching between samples and centroids, providing better controllability. As shown in Fig.~\ref{fig:Barycenter_MNIST}, by setting $\mathbf{b}^d$ and $\mathbf{b}^u$, we can constrain the number of samples in each cluster to a certain extent, thereby avoiding clusters with isolated or very few samples, as well as clusters that dominate the majority of the data;   \textbf{2)  Our} clustering methods can be applied not only to clustering feature points but also to clustering probability measures. For example, as shown in Fig.~\ref{fig:Barycenter_MNIST}, we can perform clustering on the MNIST dataset in the Wasserstein space. In this case, we no longer consider images as point vectors but treat the pixel values as histograms, and the pixel locations are used to compute the cost matrix. Moreover, this clustering method can also be used for text clustering (e.g., topic modeling), making it more adaptable to different types of datasets;  
\textbf{3)  Compared} to previous work in OT-LDA~\cite{huynh2020otlda}, we introduce a reweighting of the barycenter weights, which helps to mitigate the biases in centroid computation, as shown in Fig.~\ref{fig:clustering}.





    

\subsection{DB-OT for Long-tailed Classification}

\noindent\textbf{Motivation of using DB-OT in classification.} Here we apply DB-OT for Long-tailed unbalanced classification. Following the works~\cite{shi2023understanding,shi2023relative},  the Inverse Optimal Transport~\cite{li2019learning,stuart2020inverse,chiu2022discrete} can be set as a bi-level optimization for classification as
\begin{equation}\label{eq:iot}
        \min_\theta KL(\mathbf{\tilde{P}}|\mathbf{P}^\theta)
        \text{\quad s.t. \quad}
        \mathbf{P}^\theta = \arg \min_{\mathbf{P}\in U} <\mathbf{C}^\theta,\mathbf{P}> - \epsilon H(\mathbf{P}).
 \end{equation}
Here $U$ is the set of couplings and when the set $U=U(\mathbf{a}) = \{\mathbf{P}\in \mathbb{R}^+_{m\times n}|\mathbf{P}\mathrm{1}_m=\mathbf{a}\}$, this optimization is equivalent to:
\begin{equation}\label{eq:softce}
    \min_{\theta} \mathcal{L} = -\sum_{ij}\mathbf{\tilde{P}}_{ij} \log \left(\frac{\exp(-\mathbf{C}^\theta_{ij}/\epsilon)}{\sum_{k=1}^m\exp(-\mathbf{C}^\theta_{ik}/\epsilon))}\right)+constant,
\end{equation}
where $\mathbf{\tilde{P}}_{ij}$ is the ground truth matrix defined by the supervised one-hot labels and $\mathbf{a}=\mathbf{1}/m$. The proof process is similarly discussed in \cite{shi2023understanding}. Specifically, for the ground truth, $\mathbf{\tilde{P}}_{ij}=1$ if $j$ is the label index of sample $i$ otherwise $\mathbf{\tilde{P}}_{ij}=0$. The loss in Eq.~\ref{eq:softce} is exactly equal to the Softmax-CrossEntropy loss if we set $\mathbf{C}^\theta_{ij}=c-l_{ij}$ where $c$ is a large enough constant and $l_{ij}$ is the logits of sample $i$ on class $j$. The equivalence between Inverse OT under $U(\mathbf{a})$ and Softmax-CrossEntropy loss motivate us that we can adopt different constraints (e.g. $\mathcal{C}(\mathbf{a},\mathbf{b}^u,\mathbf{b}^d)$) instead of $U(\mathbf{a})$. Then in the following, we show our main idea that classification learning can be viewed as optimizing Inverse OT and Classification inference is learning OT.

\noindent\textbf{Training via Inverse DB-OT.}\label{sec:TrainDB-OT}
As discussed above, we apply DB-OT by setting $U=\mathcal{C}(\mathbf{a},\mathbf{b}^u,\mathbf{b}^d)$ in Eq.~\ref{eq:iot} for long-tailed classification, which assumes that the labels of training data have a known long-tailed distribution $\mathbf{r}$. We set $\mathbf{a}=1/m$ and:
\begin{equation}\label{eq:settingbubd}
\small{
        \mathbf{b}^u = (1+\delta)\mathbf{r}\quad \text{ and } \quad   \\
        \mathbf{b}^d = (1-\delta)\mathbf{r}, 
}
\end{equation}
where $\delta\in (0,1)$ is the bound rate for $\mathbf{r}$. We adopt the Bregman method for calculating $\mathbf{P}^\theta$ with $K$ iteration. We find when $K=1$ and $\delta=0$, the Balanced Softmax~\cite{ren2020balanced} is a special case of our loss, which validates the effectiveness of our theory.
\begin{table*}[tb!]
    \centering
   
    \resizebox{0.8\textwidth}{!}{
    \begin{tabular}{l||ccc||cccc||cccc}
      \toprule[1.0pt]
       \multirow{2}{*}{\textbf{\large Method}} & \multicolumn{3}{c||}{\textbf{CIFAR10-LT}} & \multicolumn{4}{c||}{\textbf{CIFAR100-LT}}  & \multicolumn{4}{c}{\textbf{ImageNet-LT}}    \\ 
       \cline{2-12}
        & Many & Few & All & Many & Medium & Few & All & Many & Medium & Few & All    \\ 
      \midrule[0.6pt]

       Vanilla Softmax & 77.4 & 68.9 &  74.9  & 75.8 & 48.2 & 11.0 & 42.0 & 57.3 & 26.2 & 3.1 & 35.0     \\
       
       LDAM ~\cite{cao2019learning} & 80.5 & 65.2 & 75.9 & 75.7 & 50.6 & 11.5 & 42.9 & \textbf{57.3} & 27.6 & 4.4 & 35.9\\
       Balanced Softmax & 82.2 & 71.6 & 79.0 & 70.3 & 50.4 & 26.5 & 47.0 & 52.5 & 38.6 & 17.8 & 41.1 \\
       Focal Loss~\cite{lin2017focal} & 79.6 & 58.4 & 73.3 & \textbf{76.1} & 46.9 & 11.1 & 41.7 & 57.3 & 27.6 & 4.4 & 35.9\\
       LogitAdjust~\cite{menon2020long} & 80.0 & 35.3 & 66.6 & 75.7 & 39.2 & 4.1 & 36.5 & 54.2 & 14.0 & 0.4 & 27.6 \\
       CB-CE~\cite{cui2019class} & 76.6 & 70.7 & 74.8 & 53.2 & 48.8 & 13.3 & 36.3 & 35.3 & 32.1 & 21.2 & 31.9\\
       CB-FC~\cite{cui2019class} & 76.6 & 70.7 & 74.8 & 53.2 & 48.8 & 13.3 & 36.3 & 35.3 & 32.1 & \textbf{21.2} & 31.9\\
       
      DB-OT Loss (ours) & \textbf{82.4} & \textbf{80.8} & \textbf{81.9} & 70.4 & \textbf{53.0} & \textbf{26.6} & \textbf{47.9} & 53.5 & \textbf{39.0} & 17.4 & \textbf{41.6} \\
      \bottomrule[1.0pt]
    \end{tabular}
    } 
     \caption{{Top-1 accuracy (\%)
   for long-tailed image classification with 200 imbalanced factor on three popular LT datasets.}} 
    \label{tab:Image_iNa_bench}
\end{table*}

\begin{table*}[tb!]
    \centering
    \resizebox{0.99\textwidth}{!}{
    \begin{tabular}{l||ccc||ccc||ccc||ccc}
      \toprule[1.0pt]
       \multirow{2}{*}{\textbf{\large Testing Inference}} & \multicolumn{3}{c||}{\textbf{Vanilla Softmax Loss}} & \multicolumn{3}{c||}{\textbf{Balanced Softmax}}  & \multicolumn{3}{c||}{\textbf{Focal Loss}}  &
       \multicolumn{3}{c}{\textbf{Logit Adjustment Loss}} \\ 
       \cline{2-13}
        & LT & Uniform & Reverse LT & LT & Uniform & Reverse LT  & LT & Uniform & Reverse LT  & LT & Uniform & Reverse LT   \\ 
      \midrule[0.6pt]

       Vanilla Softmax & 71.6 & 42.0 & 18.7 & 66.5 & 47.0 & 20.1 & 68.5 & 41.7 & 18.4 & 69.3 & 36.4 & 4.8 \\
       classifier normalize\cite{kang2019decoupling}& 70.0 & 45.2 & 16.8 & 54.9 & 42.4 & 28.6 & 67.0 & 41.3 & 13.3 & 70.5 & 40.4 & 7.9 \\
       Class-aware bias\cite{menon2020long} & 68.0 & 41.3 & 14.5 & 41.5 & 45.0 & 18.6 & 54.6 & 40.0 & 14.1 & 70.8 & 36.4 & 7.7 \\
      DB-OT inference (ours) & \textbf{71.8} & \textbf{48.5} & \textbf{36.6} & \textbf{71.4} & \textbf{47.8} & \textbf{34.6} & \textbf{68.7} & \textbf{44.1} & \textbf{34.4} & \textbf{71.3} & \textbf{44.8} & \textbf{32.3} \\
      \bottomrule[1.0pt]
    \end{tabular}
    } 
    
    \caption{Top-1 accuracy (\%) of
    CIFAR-100 for the comparison of different testing inference methods given four trained models.}\label{tab:Inference} 
\end{table*}

\noindent\textbf{Testing-time Inference with DB-OT.} 
Here we focus on using DB-OT for inference in testing process. Note our inference method is orthogonal to the training method proposed in Sec.~\ref{sec:TrainDB-OT}, which means our inference can be used for different classifiers. For the testing inference, we treat all the training as a process of feature learning and then with the feature learned, we can define the cost matrix given each batch data and match the features and labels with OT. For the long-tailed classification, the inference is as follows:
\begin{equation}
    \min_{\mathbf{P}\in \mathcal{C}(\mathbf{a},\mathbf{b}^u,\mathbf{b}^d)} <\mathbf{C},\mathbf{P}> -\epsilon H(\mathbf{P}).
\end{equation}
Here $\mathbf{C}_{ij}=c-l_{ij}$ with large enough value $c$  to guarantee the positiveness of the the cost. Meanwhile, we set $\mathbf{b}^u$ and $\mathbf{b}^d$ as Eq.~\ref{eq:settingbubd}, which enables predictions to fluctuate both upward and downward, taking into account the randomness of batch data sampling. During the testing process,  $\mathbf{r}$ can be configured as a known long-tailed distribution, uniform distribution, reverse long-tailed distribution, or any other appropriate distribution, which is a main advantage of our method. 


\section{Experiments}
We conduct experiments on both clustering and classification. To showcase the advantage of DB-OT on fine-grained controlling the behavior of clustering under bounded sizes of clusters, and classification in long-tail cases. 

\subsection{Experiments on Size-controlled Clustering}
We first study our clustering setting which each cluster's size is bounded by a certain range, using our DB-OT function with Barycenter. The experimental results on Gaussian mixture synthetic datasets with 2D 150 points are shown in Fig.~\ref{fig:clustering}. We can find that the OT-LDA and our method without reweighting exhibit a noticeable bias and re-weighting the barycenter weights can overcome this issue with Eq.~\ref{eq:reweightBC}. We also evaluate the top-1 accuracy within the cluster and the results of OT-LDA, ours (without and with reweighting) are 82.67\%, 88.00\%, and 100.00\% respectively, which shows the superiority of our method. We also do the clustering experiments on MNIST with three different methods (i.e. K-means, DB-OT under Euclidean space and DB-OT under Wasserstein space). We select parts of MNIST data (12 images in each class) and set 16 clusters for the experiments. The results are shown in Fig.~\ref{fig:Barycenter_MNIST}, and we can find that our method can be more controllable for the number of samples in each cluster. Note the bounds are to control the column probability sum of the coupling and thus it does not strictly meet the bound range for sample quantity. 

\subsection{Experiments on Long-tailed Classification}
We evaluate our two methods -- both for using DB-OT to improve the loss and adopting the DB-OT for classification inference in the testing process. These two methods are parallel to each other as the former aims to learn a better representation and the latter is to get a more accurate prediction based on a known testing label ratio. We do the experiments on CIFAR10-LT, CIFAR100-LT~\cite{krizhevsky2009learning}, ImageNet-LT~\cite{liu2019large} for image classification.  
For a fair comparison, all methods share the same network backbone and hyperparameters, including the learning rate. More detailed information about the experimental settings is given in the online Appendix.

\noindent\textbf{Improvements of DB-OT based loss.} To evaluate the performance of our DB-OT based loss, following~\cite{ren2020balanced}, we start with a vanilla Softmax pretrained model and train with our DB-OT based loss to improve the representations. We use the corresponding balanced testing dataset for evaluation, where its labels are uniformly distributed. We report top-1 accuracy as the evaluation metric. Specifically, for CIFAR10-LT, we report accuracy on two sets of classes in detail: Many-shot (more than 100 images) and Few-shot (less than 100 images). For CIFAR00-LT and ImageNet-LT, we report accuracy on three sets: Many-shot (more than 100 images), Medium-shot (20 $\sim$ 100 images), and Few-shot (less than 20 images). The experiments for unbalanced image classification are all conducted with an imbalanced factor of less than 200, which is defined as the ratio of the number of training instances in the largest class to the smallest~\cite{ren2020balanced}. The results for long-tailed classification are presented in Table~\ref{tab:Image_iNa_bench}. From a comprehensive perspective, our approach achieves the best average accuracy across the entire dataset. When examining the results for different data splits, our method particularly outperforms the others on subsets with fewer images (i.e., median or few-shot).


\noindent\textbf{Performance of testing inference.} To evaluate the performance of DB-OT based inference, we use the corresponding long-tailed (LT), uniform, and Reverse LT testing dataset for evaluation. We report top-1 accuracy as the evaluation metric. Specifically, we compare our method with vanilla Softmax, classifier normalization and class-aware bias, which are discussed in \cite{wu2021adversarial}. Table~\ref{tab:Inference} shows the comparison results for all the models and different testing data in CIFAR100. Our inference method outperforms and can achieve a great improvement when the testing data is reverse long-tailed distributed and the model trained by vanilla softmax performs the best using our DB-OT testing inference though it may fail with vanilla softmax prediction in testing. 

\noindent\textbf{Ablation study and Additional Results.} We do the ablation study by varying iterations and $\delta$ values of  Sinkhorn algorithm of DB-OT and more detailed clustering experiments comparing with more baselines in online appendix.


\section{Conclusion}
We have presented the so-called double-bounded optimal transport, with theoretical analysis and further derive three variants of algorithms to solve the problem. We then test our technique for the challenging yet realistic tasks of cluster-size bounded clustering, as well as long-tailed image recognition. Experimental results clearly verify its effectiveness.

\section{Acknowledgement}
 This work was partly supported by  NSEC (92370201, 62222607).

\bibliography{aaai24}

\newpage
\appendix
\section*{Appendix}
\section{Double Bounded on the
source and target}\label{app:SourceBounded}
Here we discuss why don't we assume that the $\alpha$ distribution is also constrained within upper and lower bounds? For instance, let's assume $\alpha^u = \sum_{i=1}^m \mathbf{a}^u_i\delta_{x_{i}}$ and $\alpha^d = \sum_{i=1}^m \mathbf{a}^d_i\delta_{x_{i}}$, and the coupling satisfies $\mathbf{a}^d\leq \mathbf{P}\mathbf{1}_n\leq \mathbf{a}^u$. In practice, the optimal transportation tends to transport mass vertically and towards the smaller lower bound. Without loss of generality, let's assume $\sum_i\mathbf{a}^d_i<\sum_j\mathbf{b}^d_j$. In this case, the optimal solution must satisfy $\sum_{ij}\mathbf{P}_{ij}=\sum_i\mathbf{a}^d_i$. 

We adopt the opposite approach. Without loss of generality, we assume that $\sum_i\mathbf{a}^d_i<\sum_{ij}\mathbf{P}_{ij}<\sum_j\mathbf{b}^d_j$ where $\mathbf{P}$ is the optimal solution satisfying $\mathbf{P}=\arg\min_{\mathbf{P}'\in \mathcal{C}}<\mathbf{C},\mathbf{P}'>$ where $ \mathcal{C}$ is the constraint set for double bounds for source and target histograms. 
Then there must exist the dimension $i$ and value $\delta_1>0$ satisfying
\begin{equation}
    \mathbf{a}^d_i +\delta_1 < \sum_{j}\mathbf{P}_{ij}.
\end{equation}
There must exist the dimension $j$ and value $\delta_2>0$ satisfying
\begin{equation}
     \sum_{i}\mathbf{P}_{ij}+\delta_2<\sum_j\mathbf{b}^d_j.
\end{equation}
Now we set $\mathbf{P}_{ij}' = \mathbf{P}_{ij}-\min(\delta_1,\delta_2)$,  the solution $\mathbf{P}'$ also satisfy $\mathbf{a}^d\leq\mathbf{P}'\mathbf{1}_n\leq \mathbf{a}^u$ and $\mathbf{b}^d\leq\mathbf{P}'^\top\mathbf{1}_m\leq \mathbf{b}^u$. However, we know $<\mathbf{C},\mathbf{P}'><<\mathbf{C},\mathbf{P}>$, which contradicts the assumption that $\mathbf{P}$ is the optimal solution.

\section{Proof and discussion on Prop.~\ref{Prop:KLproof} }\label{app:KLproof}
We rewrite the Prop.~\ref{Prop:KLproof} as:
\begin{proposition}[\textbf{Static Schrödinger Form}]
Redefine a general KL divergence in line with~\cite{benamou2015iterative}
$$
    \widetilde{KL}(\mathbf{P}|\mathbf{K})=\sum_{ij}\mathbf{P}_{ij}\log\frac{\mathbf{P}_{ij}}{\mathbf{K}_{ij}}-\mathbf{P}_{ij}+\mathbf{K}_{ij}.
$$
Let $\mathbf{K}_{ij}=e^{-\mathbf{C}_{ij}/\epsilon}$, the optimization in Eq.~\ref{eq:EDB-OT} is equivalent to the minimization:
$$\small{
        \mathbf{P}^*=\mathop{\arg\min}\limits_{\mathbf{P}\in \mathcal{C}(\mathbf{a},\mathbf{b}^u,\mathbf{b}^d)} \widetilde{KL}(\mathbf{P}|\mathbf{K}).
}
$$\end{proposition}
\begin{proof}
From the definition of $\widetilde{KL}$ and $\mathbf{K}_{ij}=e^{-\mathbf{C}_{ij}/\epsilon}$, we have
\begin{equation}
\begin{aligned}
    \min_{\mathbf{P}\in \mathcal{C}}\widetilde{KL}(\mathbf{P}|\mathbf{K})&=\min_{\mathbf{P}\in \mathcal{C}}\sum_{ij}\left( \mathbf{P}_{ij} \log\mathbf{P}_{ij} - \mathbf{P}_{ij}- \mathbf{P}_{ij}\log e^{-\mathbf{C}_{ij}/\epsilon}\right)\\
    &=\min_{\mathbf{P}\in \mathcal{C}}\sum_{ij}\left(  \mathbf{P}_{ij} \left(\log{\mathbf{P}_{ij}}-1\right)+\frac{1}{\epsilon}\mathbf{P}_{ij}\mathbf{C}_{ij}\right)\\
    & =\min_{\mathbf{P}\in \mathcal{C}} \frac{1}{\epsilon}<\mathbf{C},\mathbf{P}>-H(\mathbf{P})
\end{aligned}
\end{equation}
So the two optimization in Eq.~\ref{eq:EDB-OT} and Eq.~\ref{eq:static} are equal for $\mathcal{C}=\mathcal{C}(\mathbf{a},\mathbf{b}^u,\mathbf{b}^d)$.
\end{proof}
Then we show the proof of Eq.\ref{eq:bregman_2}, i.e.
\begin{equation}
\small{
\begin{aligned}
    & Proj_{\mathcal{C}_1}^{KL}(\mathbf{P})=\text{diag}\left(\frac{\mathbf{a}}{\mathbf{P}\mathbf{1}_n}\right)\mathbf{P}\\
    & Proj_{\mathcal{C}_2}^{KL}(\mathbf{P})=\mathbf{P}\text{diag}\left(\max\left(\frac{\mathbf{b}^d}{\mathbf{P}^\top\mathbf{1}_m},\mathbf{1}_n\right)\right)\\
    & Proj_{\mathcal{C}_3}^{KL}(\mathbf{P})=\mathbf{P}\text{diag}\left(\min\left(\frac{\mathbf{b}^u}{\mathbf{P}^\top\mathbf{1}_m},\mathbf{1}_n\right)\right),\\
\end{aligned}
}
\end{equation}
\begin{proof}
For the projection $\mathbf{P}=Proj_{\mathcal{C}}^{KL}(\mathbf{Q})$, it equals to the optimization
\begin{equation}
    \min_{\mathbf{P}\in\mathcal{C}}KL(\mathbf{P}|\mathbf{Q}),
\end{equation}
when $\mathcal{C}=\mathcal{C}_1=\{\mathbf{P}|\mathbf{P}\mathbf{1}_n=\mathbf{a}\}$, we can get the Lagrange function:
\begin{equation}
    \frac{\partial\mathcal{L}(\mathbf{P})}{\partial \mathbf{P}_{ij}}=\log\mathbf{P}_{ij}-\log \mathbf{f}_i=0,
\end{equation}
where $\mathbf{f}$ is the Lagrange variable. Then we can get
\begin{equation}
    \mathbf{P}_{ij}= e^{\mathbf{f}_i},
\end{equation}
which equals to
\begin{equation}
    \mathbf{P} = \text{diag}(e^{\mathbf{f}})\mathbf{Q}.
\end{equation}
Due to $\mathbf{P}\mathbf{1}_n=\mathbf{a}$, we can easily get
\begin{equation}
    e^{\mathbf{f}}=\frac{\mathbf{a}}{\mathbf{Q}\mathbf{1}_n},
\end{equation}
thus we can get
\begin{equation}
    \mathbf{P} = \text{diag}({\frac{\mathbf{a}}{\mathbf{Q}\mathbf{1}_n}})\mathbf{Q}.
\end{equation}
When $\mathcal{C}=\mathcal{C}_2=\{\mathbf{P}|\mathbf{P}^\top\mathbf{1}_m\geq \mathbf{b}^d\}$, introduce the Lagrange variable $\mathbf{g}$ for the optimization:
\begin{equation}
    \frac{\partial\mathcal{L}(\mathbf{P})}{\partial \mathbf{P}_{ij}}=\log\mathbf{P}_{ij}-\log \mathbf{g}_j=0,
\end{equation}
where $\mathbf{g}\geq\mathbf{0}$.Similarly, we can get
\begin{equation}
    \mathbf{P} = \mathbf{Q}\text{diag}(e^{\mathbf{g}}).
\end{equation}
When $\mathbf{g}_j > 0$, we know $(\mathbf{P}^\top \mathbf{1}_m)_j = \mathbf{b}^d_j$, then 
\begin{equation}
    e^{\mathbf{g}_j} = \frac{\mathbf{b}_j}{(\mathbf{Q}\mathbf{1}_m)_j}.
\end{equation}
While $\mathbf{g}_j = 0$, we know $(\mathbf{P}^\top \mathbf{1}_m)_j > \mathbf{b}^d_j$, then $e^{\mathbf{g}_j} =1$. Thus we have
\begin{equation}
    \mathbf{P} = \mathbf{Q}\text{diag}\left(\max\left(\frac{\mathbf{b}^d}{\mathbf{Q}^\top\mathbf{1}_m},\mathbf{1}_n\right)\right)
\end{equation}
For $\mathcal{C}=\mathcal{C}_3=\{\mathbf{P}|\mathbf{P}^\top\mathbf{1}_m\leq \mathbf{b}^u\}$, the proof is similar with the constraint set $\mathcal{C}_2$.
\end{proof}
Besides, the Bregman algorithm can be simplifed. Specifically, the first equation in Eq.~\ref{eq:bregman_2} indicates that the rows of coupling matrix $\mathbf{P}$ are normalized to match the vector $\mathbf{a}$, which means $\mathbf{P}_{i,:}\xleftarrow{}\mathbf{a}_i\cdot \mathbf{P}_{i,:}/ \sum_j \mathbf{P}_{ij}$ for each row $i$. The second equation signifies that for the $j$-th column of the coupling matrix $\mathbf{P}$, if $(\mathbf{P}^\top\mathbf{1}_m)_j<\mathbf{b}^d_j$, the corresponding column is normalized to match $\mathbf{b}^d_j$; otherwise, it remains unchanged. Similarly, in the third equation, if $(\mathbf{P}^\top\mathbf{1}_m)_j>\mathbf{b}^u_j$, the corresponding column is normalized to match $\mathbf{b}^u_j$; otherwise, it remains unchanged. So the projection in Eq.~\ref{eq:bregman_2} can also be simplified as the algorithm of row sum to $\mathbf{a}$ and column sum to $\mathbf{b}$ algorithms given the iteration number $l$:
\begin{equation}
\small{
\begin{aligned}
 &\mathbf{P}^{(l+1/2)}=    \text{diag}\left(\frac{\mathbf{a}}{\mathbf{P}^{(l)}\mathbf{1}_n}\right)\mathbf{P}^{(l)},   \\
 &\mathbf{P}^{(l+1)}=\mathbf{P}^{(l+1/2)}\text{diag}\left(\frac{\mathbf{b}}{(\mathbf{P}^\top)^{(l+1/2)}\mathbf{1}_m}\right),
\end{aligned}
}
\end{equation}
where $\mathbf{b}$ is specified as
\begin{equation}
\small{
    \mathbf{b}_j=\left\{
    \begin{aligned}
        &\mathbf{b}_j^d, & &  ((\mathbf{P}^\top)^{(l+1/2)}\mathbf{1}_m)_j\leq\mathbf{b}^d_j,\\
        &\mathbf{b}_j^u, & &  ((\mathbf{P}^\top)^{(l+1/2)}\mathbf{1}_m)_j\geq\mathbf{b}^u_j,\\
        &((\mathbf{P}^\top)^{(l+1/2)}\mathbf{1}_m)_j, & & \text{otherwise}.
    \end{aligned}
    \right.
    }
\end{equation}

\begin{figure*}[tb!]
    \centering
    \includegraphics[width=0.9\textwidth]{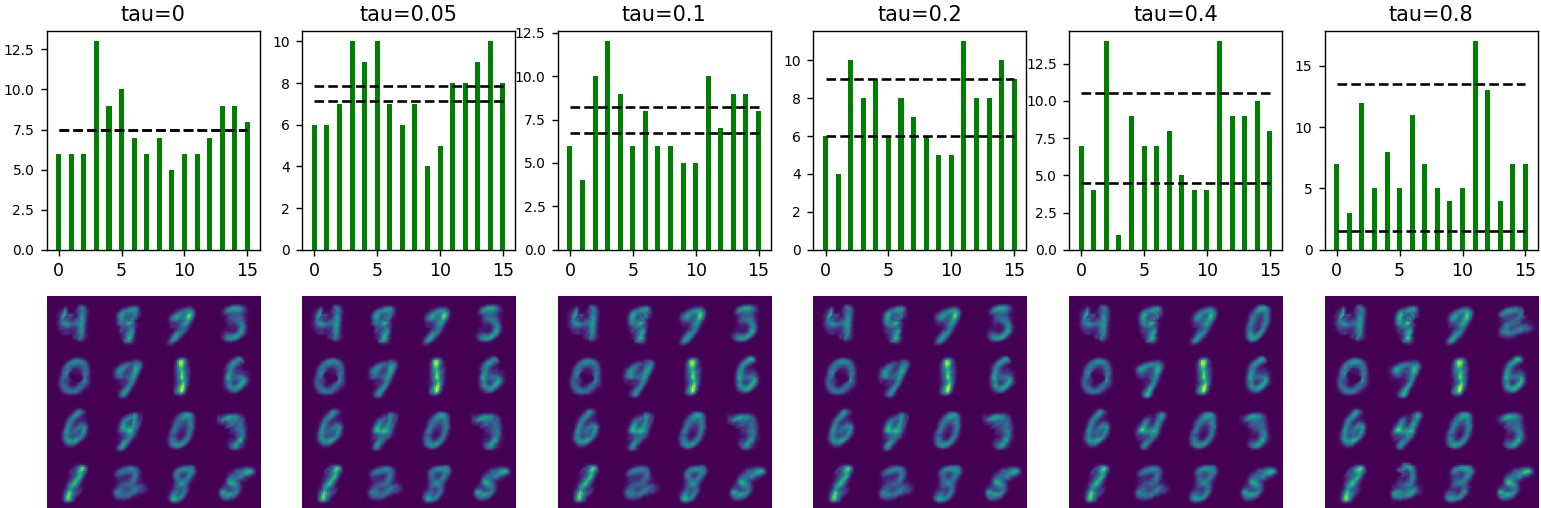}
    \caption{The result of clustering with different bounds. The top-1 accuracy is 72.50, 70.83, 75.00, 70.00, 70.00, 68.33 respectively. The six histograms indicate the number of each class and the balck dotted line is the bound of each case.
    } 
    \label{fig:diffrentTauMu}
\end{figure*}

\section{Proof and discussion on Prop.~\ref{prop:uKv}}\label{app:uKv}
We rewrite the Prop.~\ref{prop:uKv} as
\begin{proposition}[\textbf{Solution Property}]
The optimal solution of the optimization in Eq.~\ref{eq:EDB-OT} is unique and has the form 
$$
    \mathbf{P}^* = \text{diag}(\mathbf{u})\mathbf{K}\text{diag}(\mathbf{q}\odot\mathbf{v})
$$
where $\mathbf{q}\odot\mathbf{v}=(\mathbf{q}_i\mathbf{v}_i)\in \mathbb{R}_n^+$ and the three scaling variables $(\mathbf{u},\mathbf{q},\mathbf{v})$ satisfy $\mathbf{u}\in R^+_m$, $\mathbf{0}_n\leq\mathbf{v}\leq\mathbf{1}_n$ and $ \mathbf{q}\geq \mathbf{1}_n$.
\end{proposition}
\begin{proof}
The minimization of Eq.~\ref{eq:EDB-OT} can be specified with Lagrange method:
\begin{equation}
\small{
    \mathcal{L} = <\mathbf{C},\mathbf{P}>-\epsilon H(\mathbf{P})-\mathbf{f}(\mathbf{P}\mathbf{1}_n-\mathbf{a})-\mathbf{g}(\mathbf{P}^\top \mathbf{1}_m -\mathbf{b}^d)-\mathbf{h}(\mathbf{P}^\top \mathbf{1}_m -\mathbf{b}^u)
}
\end{equation}
where $\mathbf{f}\in \mathbb{R}^n$, $\mathbf{g}\geq \mathbf{0}$ and $\mathbf{h}\leq \mathbf{0}$. Then we can get
\begin{equation}
    \frac{\partial \mathcal{L}}{\partial \mathbf{P}_{ij}}= \log\mathbf{C}_{ij}-\epsilon \log \mathbf{P}_{ij}-\mathbf{f}_i -\mathbf{g}_j-\mathbf{h}_j=0.
\end{equation}
Then we can get
\begin{equation}
    \mathbf{P}=\text{diag}(e^{\mathbf{f}/\epsilon})e^{-\mathbf{C}/\epsilon}\text{diag}(e^{(\mathbf{g}+\mathbf{h})/\epsilon}).
\end{equation}
Due to $\mathbf{g}\geq \mathbf{0}$ and $\mathbf{h}\leq \mathbf{0}$, we set
\begin{equation}
    \begin{aligned}
     & \mathbf{u}= e^{\mathbf{f}/\epsilon}\\
     & \mathbf{q}= e^{\mathbf{g}/\epsilon}\geq \mathbf{1}_n\\
     & \mathbf{v}= e^{\mathbf{h}/\epsilon}\leq \mathbf{1}_n
    \end{aligned}
\end{equation}
then we can get the optimal solution satisfying
\begin{equation}\label{eq:ukqv}
    \mathbf{P} = \text{diag}(\mathbf{u})\mathbf{K}\text{diag}(\mathbf{q}\odot\mathbf{v}).
\end{equation}
\end{proof}
Then we prove Eq.~\ref{eq:sinkhorn2}, i.e.
\begin{equation}
\small{
\begin{aligned}
    &\mathbf{u}^{(l+1)} =  \frac{\mathbf{a}}{K(\mathbf{v}^{(l)}\odot\mathbf{q}^{(l)})},  \\
    & \mathbf{q}^{(l+1)} = 
    \max\left(\frac{\mathbf{b}^d}{\mathbf{u}^{(l+1)}K\odot\mathbf{v}^{(l)}},\mathbf{1}_n\right),\\
    &\mathbf{v}^{(l+1)} =  \min\left(\frac{\mathbf{b}_u}{\mathbf{u}^{(l+1)}K\odot\mathbf{q}^{(l+1)}},\mathbf{1}_n\right).
\end{aligned}
}
\end{equation}
\begin{proof}
Due to Eq.~\ref{eq:ukqv} and the constraints $\mathbf{P}\mathbf{1}_n=\mathbf{a}$  , we know
\begin{equation}
    \mathbf{u} \odot (\mathbf{K}(\mathbf{q}\odot\mathbf{v}))=\mathbf{a},
\end{equation}
thus we have 
\begin{equation}
    \mathbf{u} = \frac{\mathbf{a}}{\mathbf{K}(\mathbf{q}\odot\mathbf{v})}
\end{equation}
And we also have
$\mathbf{b}^d\leq \mathbf{P}^\top\mathbf{1}_m\leq\mathbf{b}^u$. 
When $\mathbf{q}_j=1$, we can get $\mathbf{g}_j=0$ and $(\mathbf{P}\mathbf{1}_m)_j>\mathbf{b}^d_j$, while  $\mathbf{q}_j>1$, we can get $\mathbf{g}_j>0$ and $(\mathbf{P}\mathbf{1}_m)_j=\mathbf{b}^d_j$, then
\begin{equation}
    (\mathbf{q}\odot\mathbf{v})_j\cdot(\mathbf{K}^\top\mathbf{u})_j=\mathbf{b}^d_j.
\end{equation}
Thus we have
\begin{equation}
    \mathbf{q}_j = \frac{\mathbf{b}^d_j}{\mathbf{v}_j\cdot(\mathbf{K}^\top\mathbf{u})_j}.
\end{equation}
Combining the situation for $\mathbf{q}_j=1$, we have
\begin{equation}\label{eq:qeq}
     \mathbf{q} = \max(\frac{\mathbf{b}^d}{\mathbf{v}\odot(\mathbf{K}^\top\mathbf{u})},\mathbf{1}_n)
\end{equation}
The proof of equation
\begin{equation}
    \mathbf{v} = \max(\frac{\mathbf{b}^u}{\mathbf{q}\odot(\mathbf{K}^\top\mathbf{u})},\mathbf{1}_n)
\end{equation}
is similar to the proof of $\mathbf{q}$ in Eq.~\ref{eq:qeq}. We do not repeat here.
\end{proof}

\section{Proof and discussion on Prop.~\ref{Prop:dual}}\label{app:dual}
We rewrite the Prop.~\ref{Prop:dual} here.
\begin{proposition}[\textbf{Dual Formulation}]
From the  optimization in Eq.~\ref{eq:EDB-OT}, we can get its dual formulation by maximizing
\begin{equation}
     L = <\mathbf{f},\mathbf{a}>+<\mathbf{g},\mathbf{b}^d> +  <\mathbf{h},\mathbf{b}^u> - \epsilon B(\mathbf{f},\mathbf{g},\mathbf{h})
\end{equation}
where $\mathbf{f}\in \mathbb{R}^n$, $\mathbf{g}\geq  \mathbf{0}$ and $\mathbf{h}\leq\mathbf{0}$ are the corresponding dual variables and $B(\mathbf{f},\mathbf{g},\mathbf{h}) = <e^{\mathbf{f}/\epsilon},\mathbf{K}e^{(\mathbf{g}+\mathbf{h})/\epsilon}>$.
\end{proposition}

\begin{proof}
For the coupling, we have
\begin{equation}
    \mathbf{P}_{ij} = e^{\mathbf{f}_i/\epsilon}e^{-\mathbf{C}_{ij}/\epsilon}e^{(\mathbf{g}_j+\mathbf{h}_j)/\epsilon},
\end{equation}
where $\mathbf{f}\in \mathbb{R}^n$, $\mathbf{g}\geq \mathbf{0}$ and $\mathbf{h}\leq \mathbf{0}$.Then the objective function is specified as
\begin{equation}\label{eq:fckgh}
    <e^{\mathbf{f}/\epsilon},\mathbf{C}\odot\mathbf{K}e^{(\mathbf{g}+\mathbf{h})/\epsilon}> - \epsilon H(\mathbf{P})
\end{equation}
where $-\epsilon H(\mathbf{P})$ is
\begin{equation}
    \begin{aligned}
     -\epsilon H(\mathbf{P})=& \epsilon<\mathbf{P},\log\mathbf{P}-\mathbf{1}_{n\times m}> \\
     =&
        <\mathbf{P},\mathbf{f}\mathbf{1}^\top_m+\mathbf{1}_n(\mathbf{g}+\mathbf{h})^\top-\mathbf{C}-\epsilon \mathbf{1}_{n\times m}>\\
    = &- <e^{-\mathbf{f}/\epsilon},\mathbf{C}\odot\mathbf{K}e^{-(\mathbf{g}+\mathbf{h})/\epsilon}>+ <\mathbf{f},\mathbf{a}>\\ &+<\mathbf{g},\mathbf{b}^d>+<\mathbf{h},\mathbf{b}^u>+
    <e^{\mathbf{f}/\epsilon},\mathbf{K}e^{(\mathbf{g}+\mathbf{h})/\epsilon}>
    \end{aligned}.
\end{equation}
So the following holds where $B(\mathbf{f},\mathbf{g},\mathbf{h})=<e^{\mathbf{f}/\epsilon},\mathbf{K}e^{(\mathbf{g}+\mathbf{h})/\epsilon}>$:
\begin{equation}
\begin{aligned}
    & \min_{\mathbf{P}\in \mathcal{C}} <\mathbf{C},\mathbf{P}>-\epsilon H(\mathbf{P})\\
   = &\max_{\mathbf{f},\mathbf{g}\geq \mathbf{0},\mathbf{h}\leq \mathbf{0}}  <\mathbf{f},\mathbf{a}> +  <\mathbf{g},\mathbf{b}^d>+<\mathbf{h},\mathbf{b}^u> - \epsilon B(\mathbf{f},\mathbf{g},\mathbf{h})  
\end{aligned}
\end{equation}
Therefore, the first term in Eq.~\ref{eq:fckgh} cancels out with the first term in the entropy above.
The remaining terms are those appearing in Eq.~\ref{eq:dual}.
\end{proof}
\section{Experimental Details}\label{app:ExperimentDetails}
\subsection{Hardware and Software}
We use Intel Core i9-10920X CPU @ 3.50GHz with Nvidia GeForce RTX 2080 Ti GPU for model training. We take single GPU to train models on CIFAR-10-LT, CIRFAR-100-LT, ImageNet-LT. We implement our proposed algorithm with PyTorch-2.0.1 for all experiments. 

\subsection{More Experimental Setting Details and Results for Image Classification}

We perform the long-tailed image classification task on CIFAR10-LT, CIFAR100-LT~\cite{krizhevsky2009learning}, and Imagenet-LT~\cite{liu2019large} datasets, and evaluate on balanced testing data by reporting top-1 accuracy. For CIFAR10 and CIFAR100, the experiments of image classificaiton tasks are based on ResNet32~\cite{he2016deep} as the backbone with 0.02 and 0.005 learning rate respectively, while for Imagenet dataset, we use ResNet10 for training with 0.005 learning rate. We train the CIFAR10 and CIFAR100 data with 20000 and 30000 iterations on a single GPU and imbalanced ratio is set as 200, 100, 10, resp. For testing inference, we construct the function of Softmax, Classifier Normalization, Class Aware Bias and ours method to evaluate the accuracy on the testing set. We set the $\epsilon$ of sinkhorn as 1 to fit the training function, \textbf{a} as all one vector and use the number of each class in testing set to determine the \textbf{b} as sinkhorn bounds.

\begin{table}[tb!]
    \centering
    \caption{{Top-1 accuracy ($\%$)
   for CIFAR10-LT image classification with different inference methods. }} \label{tab:infer}
    \resizebox{0.48\textwidth}{!}{
    \begin{tabular}{l||ccc||ccc}
      \toprule[1.0pt]
       \multirow{2}{*}{\textbf{\large Test Inference}} & \multicolumn{3}{c}{\textbf{\large DB-OT loss}} &
       \multicolumn{3}{c}{\textbf{\large CBCE}}\\ 
       \cline{2-7}
        & LT & Uniform & Revers LT & LT & Uniform & Revers LT\\ 
       \midrule[0.6pt]
        Vanilla Softmax & 85.3 & 79.7 & 64.1 & 6.2 & 54.6 & 91.6 \\
        classifier normalize\cite{kang2019decoupling}& 79.1 & 81.1 & 73.7 & 6.5 & 54.9 & 91.6 \\
        Class-aware bias\cite{menon2020long} & 46.5 & 79.7 & 35.8 & 4.8 & 54.6 & 91.7 \\
        DB-OT inference (ours) & \textbf{89.5} & \textbf{81.9} & \textbf{82.4} & \textbf{65.5} & \textbf{67.8} & \textbf{92.1}\\
      \bottomrule[1.0pt]
    \end{tabular}
    }
    \label{tab:}
\end{table}

\subsection{More Experimental Setting Details and Results for Clustering}

We perform the clustering task on the subset of MNIST(select 12 images each class) and evaluate the efficiency by top-1 accuracy. We set the training iteration as 5 and the inner iteration of sinkhorn has the maximum number of 1000. As for training, we use the similar method as K-means to generate the initial clustering centers. We also construct the evaluating function to show the result of clustering and calculate the top-1 accuracy.

\section{Baselines of Inference in Testing}\label{app:InferenceBaselines}

\subsubsection{Classifier Normalization.} This inference method is proposed in \cite{kang2019decoupling}, which reset the logits in testing as
\begin{equation}
    h_j = (W_j/||W_j||^\tau)^\top f(x),
\end{equation}
where $\tau$ is the hyper-parameter,  $f(x)$ is the feature of sample $x$ and $W_j$ is the linear projection to the class $j$. Then we use the softmax based on $h_j$ as the prediction results.
\begin{table}[tb!]
    \centering
    \caption{\small{Top-1 accuracy (\%)
   for CIFAR10-LT image classification with different training and testing iteration. }} 
    \resizebox{0.48\textwidth}{!}{
    \begin{tabular}{l||c||c||c||c||c||c||c}
      \toprule[1.0pt]
       \multirow{2}{*}{\textbf{\large Training iteration}} & \multicolumn{7}{c}{\textbf{\large Testing iteration}} \\ 
       \cline{2-8}
        & iter=10 & iter=30 & iter=50 & iter=70 & iter=100 & iter=150 & final acc (final iter) \\ 
       \midrule[0.6pt]
        iter=1 & 75.2 & 76.0 & 75.7 & 75.6 & 77.1 & 75.0 & 76.4 (160) \\
        iter=2 & 80.1 & 80.8 & 80.0 & 80.3 & 81.0 & 80.8 & 80.7 (250) \\
        iter=3 & 83.6 & 83.4 & 83.1 & 83.1 & 83.3 & 84.1 & 83.5 (220) \\
        iter=4 & 82.3 & 82.2 & 82.5 & 82.5 & 83.2 & 83.6 & 82.0 (230) \\
        iter=5 & 78.5 & 78.2 & 75.5 & 76.9 & 77.3 & 77.4 & 77.8 (170) \\
      \bottomrule[1.0pt]
    \end{tabular}
    }
    \label{tab:differentIteration}
\end{table}
\subsubsection{Classifier Normalization.} This inference method is proposed in \cite{menon2020long}, which set the new logits in testing process as
\begin{equation}
     h_j = W_jf(x)-\tau\log n_j
\end{equation}
where $n_j$ corresponds to the ratio of long-tailed ratio. Tab.\ref{tab:infer} shows the results on CIFAR-10.

\section{Ablation Study}\label{app:Ablation}

\subsection{Ablation of Varying Bounds on Clustering}
The Fig~.\ref{fig:diffrentTauMu} is the centroids of clustering on MNIST using Wasserstein distance with different $\tau$, which shows the influence of varing bounds on clustering. Each figure has the same $\epsilon=0.001$ and the training iteration as 5. 

In addition, we also do experiment on not uniform bounds and evaluate the top-1 accuracy. The results of linear increasing $\tau$ (first number is 0 and the last is 15/16 while the whole vector is a arithmetic progression) and linear decreasing $\tau$ (the inverse list of linear increasing $\tau$) is 67.50 and 71.67 respectively. We also set the list of $\tau$ as a 16 number of Gaussian, whose sum is 10, with the middle point 8 and the STD of the list is 5. The top-1 accuracy under this bound is 67.50 while under the corresponding Gaussian list (1 minus the previous list) the accuracy is 70.00.

\subsection{Ablation of Varying Iteration on Classification}
Because there are sinkhorn function in training and inference, we do experiment with different iteration on CIFAR10 to find the influence that varying iteration may cause.

Tab.\ref{tab:differentIteration} is the result of the experiment. We set the iteration in training loss from 1 to 5 and the iteration in inference as 10, 30, 50, 70, 100, 150. As the sinkhorn in DB-OT method will stop when the change of the norm of iterative matrix is smaller than a specific number, the last column is the final accuracy with the stop iteration in the brackets.

\subsection{Ablation of Varying Delta on Classification}
In order to find the influence of varying delta in both training loss and the inference function, we try different delta in these two function and evaluate the top-1 accuracy on CIFAR10 to see the difference.

From Tab.~\ref{tab:AblationCIFAR10Delta}, it's obvious that the top-1 accuracy reduces when the inference delta increases, while the accuracy will reach the highest when training delta between 0.1 to 0.4.

\begin{table*}[tb!]
    \centering
    \caption{\small{Top-1 accuracy (\%) of
    CIFAR-10 for the comparison of different inference delta given several trained models with different delta.}} 
    \resizebox{0.99\textwidth}{!}{
    \begin{tabular}{l||cc||cc||cc||cc||cc||cc}
      \toprule[1.0pt]
       \multirow{2}{*}{\textbf{\large models or delta of training}} & \multicolumn{2}{c||}{\textbf{Inference  $\delta = $0}} & \multicolumn{2}{c||}{\textbf{Inference  $\delta = $0.05}}  & \multicolumn{2}{c||}{\textbf{Inference  $\delta = $0.1}}  &
       \multicolumn{2}{c||}{\textbf{Inference  $\delta = $0.2}} &
       \multicolumn{2}{c||}{\textbf{Inference  $\delta = $0.4}} &
       \multicolumn{2}{c}{\textbf{Inference  $\delta = $0.8}} \\ 
       \cline{2-13}
        & All & Many & All & Many & All & Many & All & Many & All & Many & All & many\\ 
      \midrule[0.6pt]
      Balanced Softmax & 81.0 & 81.5 & 81.0 & 82.2 & 80.8 & 83.0 & 80.2 & 83.9 & 79.3 & 83.4 & 79.2 & 83.3 \\
       Focal Loss & 79.3 & 78.9 & 79.3 & 79.6 & 79.0 & 79.9 & 78.2 & 80.6 & 75.3 & 80.4 & 72.7 & 79.7 \\

       Ours($\delta=0$) & 81.8 & 81.8 & 81.7 & 81.8 & 81.5 & 81.8 & 81.0 & 81.2 & 80.3 & 80.2 & 80.3 & 80.2 \\
       Ours($\delta=0.05$) & 82.4 & 82.6 & 82.2 & 82.8 & 81.9 & 83.3 & 81.1 & 83.5 & 80.4 & 83.4 & 80.4 & 83.4 \\
       Ours($\delta$=0.1) & 82.5 & 82.8 & 82.4 & 83.5 & 82.2 & 84.0 & 81.3 & 84.4 & 80.5 & 84.7 & 80.5 & 84.7 \\
       Ours($\delta$=0.2) & 82.7 & 83.4 & 82.5 & 83.9 & 82.1 & 84.2 & 81.5 & 84.7 & 79.9 & 85.0 & 79.9 & 85.0 \\
       Ours($\delta$=0.4) & 82.3 & 82.0 & 82.2 & 82.6 & 81.6 & 82.7 & 81.1 & 83.1 & 81.0 & 83.2 & 81.0 & 83.2 \\
       Ours($\delta$=0.8) & 81.9 & 82.2 & 81.7 & 82.6 & 81.5 & 83.0 & 80.4 & 83.6 & 79.4 & 83.5 & 79.4 & 83.5 \\
       
      \bottomrule[1.0pt]
    \end{tabular}
    } 
    \label{tab:AblationCIFAR10Delta}
\end{table*}


\section{Proof of Eq.~\ref{eq:softce}}\label{app:ua}
Now we show the softmax with the constraints:
\begin{equation}
        U(\mathbf{a}) = \{\mathbf{P}\in \mathbb{R}_+^{m\times n }|\mathbf{P}\mathrm{1}_n=\mathbf{a}\}
\end{equation}
where $\mathbf{a}=\mathbf{1}/m$ and $\mathbf{1}_m$ is the $m$-dimensional  column vector whose elements are all ones. With the objective of the entropic OT:
\begin{equation}
    \mathbf{P}^\theta = \arg\min_{P\in U(\mathbf{a})} <\mathbf{C}^\theta,\mathbf{P}> -\epsilon H(\mathbf{P}),
\end{equation}
We introduce the dual variable $\mathbf{f}\in R^m$. The Lagrangian of the above equation is:
\begin{equation}
    L(\mathbf{P},\mathbf{f})=<\mathbf{C},\mathbf{P}> -\epsilon H(\mathbf{P}) -\sum_{i=1}^n\mathbf{f}_i \cdot\left(\sum_{j=1}^n \mathbf{P}_{ij}-\frac{1}{m}\right)
\end{equation}
The first order conditions then yield by:
\begin{equation}
    \frac{\partial L(\mathbf{P},\mathbf{f})}{\partial \mathbf{P}_{ij}}=\mathbf{C}_{ij}+\epsilon\log{\mathbf{P}_{ij}}-\mathbf{f}_i = 0.
\end{equation}
Thus we have $\mathbf{P_{ij}}=e^{(\mathbf{f}_i-C^\theta_{ij})/\epsilon}$ for every $i$ and $j$, for optimal $\mathbf{P}$ coupling to the regularized problem. 
Due to $\sum_j \mathbf{P}_{ij}=1/m$ for every $i$, we can calculate the Lagrangian parameter $\mathbf{f}_i$ and the solution of the coupling is given by:
\begin{equation}\label{eq:pij0}
        \mathbf{P}_{ij}=\frac{\exp{(-\mathbf{C}^\theta_{ij}/\epsilon})}{m\sum_{k=1}^m\mathbf{Q}_{ik}\exp{(-\mathbf{C}^\theta_{ik}/\epsilon)}}.
\end{equation}
 Then in outer minimization, we have
 \begin{equation}
    \min_{\theta} {L} = -\sum_{ij}\mathbf{\tilde{P}}_{ij} \log \left(\frac{\exp(-\mathbf{C}^\theta_{ij}/\epsilon)}{\sum_{k=1}^m\exp(-\mathbf{C}^\theta_{ik}/\epsilon))}\right)+constant,
\end{equation}

\end{document}